\documentclass{article}
\usepackage{ijcai19}

\usepackage{times}
\usepackage{soul}
\usepackage{url}
\usepackage[hidelinks]{hyperref}
\usepackage[utf8]{inputenc}
\usepackage[small]{caption}
\usepackage{graphicx}
\graphicspath{{figures/}}
\DeclareGraphicsExtensions{.pdf,.jpeg,.png}
\usepackage{chngpage}
\usepackage{calc}
\usepackage{amsmath}
\usepackage{amsfonts}
\usepackage{booktabs}
\usepackage{algorithm}
\usepackage{algorithmic}
\usepackage{stmaryrd}
\usepackage{latexsym} 
\usepackage{amsthm}

\newtheorem{theorem}{Theorem}
\newtheorem{lemma}[theorem]{Lemma}
\newtheorem{corollary}{Corollary}[theorem]
\newtheorem{conj}[theorem]{Conjecture}

\title{Layer Dynamics of Linearised Neural Nets}

\author{
Saurav Basu\and
Koyel Mukherjee\and
Shrihari Vasudevan\\
\affiliations
IBM Research, INDIA
\emails
\{saubasu1, kmukherj, shrivasu\}@in.ibm.com
}

\begin{document}

\maketitle

\begin{abstract}
  Despite the phenomenal success of deep learning in recent years, there remains a gap in understanding the fundamental mechanics of neural nets. More research is focussed on handcrafting complex and larger networks, and the design decisions are often ad-hoc and based on intuition. Some recent research has aimed to demystify the learning dynamics in neural nets by attempting to build a theory from first principles, such as characterising the non-linear dynamics of specialised \textit{linear} deep neural nets (such as orthogonal networks). In this work, we expand and derive properties of learning dynamics respected by general multi-layer linear neural nets. Although an over-parameterisation of a single layer linear network, linear multi-layer neural nets offer interesting insights that explain how learning dynamics proceed in small pockets of the data space. We show in particular that multiple layers in linear nets grow at approximately the same rate, and there are distinct phases of learning with markedly different layer growth. We then apply a linearisation process to a general RelU neural net and show how nonlinearity breaks down the growth symmetry observed in LNNs. Overall, our work can be viewed as an initial step in building a theory for understanding the effect of layer design on the learning dynamics from first principles.
  \end{abstract}

\section{Introduction}\label{sec:intro}

Deep learning has admittedly been a breakthrough in the recent history of AI, and vigorous research is directed towards building progressively complex networks that can train larger and more unstructured data such as video and text \cite{alexnet,imagenet-1k,resnet,lstm}. However, fundamental understanding of the mechanics of learning in deep neural nets is in many cases elusive \cite{Combes2018}. Theoretical explanation of the effect of design parameters in neural nets such as layer depth and layer dimension on convergence and learning speed are in preliminary stages.

A mature theoretical foundation for training dynamics is essential to handcrafting complex networks. Admittedly, closed form solutions for the evolution of learning in general neural nets (GNN) is notoriously difficult. Hence, an intuitive way of following the dynamics of neural nets is by developing theories for simpler surrogate networks such as LNNs and extending it to more complex GNNs \cite{saxe1,arora1}.

Multi-layer LNNs provide an over-parameterisation of a single layer linear network in the sense that the expressivity of the network is still limited (the input-output relationship can be described by a single matrix multiplication). However, due to the multiplication of weight matrix elements in the multilayer networks, the cost function (empirical squared loss, for instance) becomes a non-linear and non-convex function of the weight matrices. 

Many interesting phenomenon in GNNs such as periods of sharp learning followed by slow learning, and dependence of convergence on layer depth and dimensions, can be observed in LNNs and might possibly be due to the non-linearity introduced by multiplication of weights. Additionally, reasonable linearisation of non-linear activation functions in GNNs allow us to view all GNNs operating in a linear mode in different pockets of the data space. Therefore, much insight can be derived from the characterisation of theoretical properties of LNNs, and techniques can be developed to incorporate those insights for \textit{linearised} neural nets as a proxy for GNNs.

Earlier and some recent work on LNNs  \cite{hornik1,kawaguchi1,fukumizu1,Choromanska15a,HardtM16,PoggioL17,Soudry1} have focussed on the characterisation of the shape of the squared empirical loss surface in linear networks. Some remarkable properties of such loss surfaces are that all local minima are global minima, and \cite{kawaguchi1} show that loss surfaces of non-linear networks also can be reduced to loss surfaces of linear networks. The common observation that arbitrary initialisations all converge to good minima in deep learning can be explained with such characterisations. 

More recently, there has been a resurgence in the analysis of LNNs \cite{saxe1,advani1,arora1,Goodfellow-et-al-2016,Poggio_puzzle,Poggio_puzzle_2} from the point of view of learning dynamics. For example, \cite{saxe1} quantify the evolution of the singular values of different layer matrices in orthogonally initialised linear networks. They show that in such networks, the different modes of input-output correlation and the layer weights balance each other in a decoupled manner independent of other modes. 

\cite{arora1} show that in a general linear deep network, the effect of the over-parameterisation of multiple layers is to accelerate convergence, and the same effect cannot be obtained by simple regularisation of the cost function. Depth acts as a pre-conditioner to the update step. 

\cite{Poggio_puzzle,Poggio_puzzle_2} recently showed that the gradient descent dynamics in a non-linear network is topologically equivalent to a linear gradient system with quadratic potential with a degenerate Hessian. They extend this result to show that similar to gradient descent in linear networks, there is an early stopping criterion to avoid overfitting.

In short, many defining properties of deep learning can be observed in multilayer linear networks, and possibly arise because of the same phenomenon of coupling of weight elements in both networks. Therefore, a solid theoretical understanding of general LNNs seems like a practical and logical step to have a fundamental understanding of deep learning. In this work, we illustrate how layer growth rates are related to each other, and quantify evolution of the layer weights across time. In addition, we provide directions on how multilayer LNN theory can be used to construct a theory for GNNs. 


 The rest of the paper is organised as follows - section \ref{sec:lnn} sets the background for LNNs and describes relationships between layers during training. Section \ref{sec:gnn} describes a particular linearisation technique for GNNs and extends insights from section \ref{sec:lnn} to GNNs. Section \ref{sec:expts} provides a brief empirical study to validate some claims from sections \ref{sec:lnn} and \ref{sec:gnn}, which is followed by conclusion in section \ref{sec:concl}.

\section{Linear Neural Networks}\label{sec:lnn}

In this section, we will give a brief overview of general LNNs, and then derive various properties of such networks. 

We define $\mathcal{X} \doteq \mathbb{R}^d$ as the feature space (say a space of images or text embeddings), and $\mathcal{Y} \doteq \mathbb{R}^k$ as the output space (for example, the weight of the $k$ possible inferences). Let the learning algorithm learn a prediction function $\hat{y} = f(x; W), x\in \mathcal{X}, y \in \mathcal{Y}, W \in \mathcal{W} \subset \mathbb{R}^D$, where $W$ is the $D$-dimensional parameterisation of the learning function. Given a point-wise loss function $l(\hat{y},y)$ that measures discrepancies between predicted and real labels, and a training set $\{x_i, y_i\}_{i=1}^m$ of features and corresponding predictions, the learning algorithm learns $f$ by estimating 

\begin{equation} 
W = \arg \underset{W \in \mathcal{W}}{\mathrm{min}} \Big[ \frac{1}{m} \sum_{i=1}^m l(f(x_i; W),y_i) \Big]
\label{eq:learning_algo}
\end{equation}

In case of a general $L+1$ layer neural network in it's simplest form, the parameterisation $W$ consists of $L$ weight matrices $\{W_1, \cdots, W_L\}$ and $L$ biases $\{b_1, \cdots, b_L\}$, and together with a non-linear function $\sigma(.)$ (where $\sigma: \mathbb{R} \rightarrow \{0,1\}$ in case of sigmoid, tanh, or $\sigma: \mathbb{R} \rightarrow \{0,\infty \}$ in case of RelU), the predictor function $f(.)$ takes the form

\begin{equation} 
f(x,W) = \sigma(W_L \sigma(W_{L-1} (\cdots \sigma(W_1 x + b_1) \cdots )+ b_{L-1}) + b_L)
\label{eq:gnn}
\end{equation}

(\ref{eq:gnn}) can be further simplified in case of a LNN, and ignoring biases, the linear predictor function can be written as 

\begin{equation} 
\hat{y} = f(x,W) = W_L W_{L-1} \cdots W_1 x
\label{eq:lnn}
\end{equation}.

In this context, a multi-layer LNN with depth $L, (L \geq 2)$ will have weights $\{W_1, \cdots, W_L\}$ with dimensions $\{n_2 \times n_1, n_3 \times n_2, \cdots, n_{L+1} \times n_{L}\}$ with $n_1 = d$ and $n_{L+1} = k$. For a general loss function $l(f(x; W),y)$, the \textit{gradient descent} or GD step for iterative minimisation of the cumulative loss $\mathcal{L} = \frac{1}{m} \sum_{i=1}^m l(f(x_i; W),y_i)$ in (\ref{eq:learning_algo}) is given as

\begin{equation} 
W_l^{t+1} \mapsfrom W_l^{t} - \eta \frac{\partial \mathcal{L} }{\partial W_l^{t}}
\label{eq:gd}
\end{equation}

where $\eta > 0$ is the \textit{learning rate}, $l$ denotes layer number and $t$ is the iteration index. A popular practice is to study iterative updates such as (\ref{eq:gd}) from the point of view of \textit{ordinary differential equations} (ODEs) by assuming $\eta \rightarrow 0$ \cite{ode1,ode2}; such an approximation allows well studied theory of ODEs to be applied to GD for gaining critical insight. Defining $\tau = \eta^{-1}$, and via standard \textit{backpropagation} (for example, see \cite{arora1}) we can show that for a LNN, the derivative of the $l$-th weight matrix in (\ref{eq:gd}) is given by -

%
%
%

\begin{equation} 
\tau \frac{d W_l}{dt} = - \Big[\prod_{j=l+1}^{L } W_j \Big]^T \Big[\frac{1}{m} \sum_{i=1}^m \frac{\partial l(\hat{y}_i,y_i)}{ \partial \hat{y}_i} x_i^T  \Big]  \Big[\prod_{j=1}^{l -1} W_j \Big]^T.
\label{eq:gd_ode_lnn}
\end{equation}

\noindent The special form of (\ref{eq:gd_ode_lnn}) offers a very interesting symmetry in the weight updates of adjacent layers of an LNN, which we prescribe below

\begin{theorem}
For a multilayer linear neural network as described in (\ref{eq:lnn}), for all $t$ and $l \in \{1,\cdots, L\}$, we have \\
\begin{equation} 
 \frac{d }{dt} (W_{l+1}^T W_{l+1}) =  \frac{d}{dt} (W_{l} W_{l}^T)
\end{equation}
\label{thm:sym}
\end{theorem}

\begin{proof}
Right multiplying both sides of the $l$-th equation of (\ref{eq:gd_ode_lnn}) by $W_l^T$ and left multiplying both sides of the $l+1$-th equation of (\ref{eq:gd_ode_lnn}) by $W_{l+1}^T$ and comparing the RHS of both equations, we get: \\
\begin{equation} 
\frac{d W_{l}}{dt} W_{l}^T = W_{l+1}^T \frac{d W_{l+1}}{dt} 
\label{eq:sym_proof_1}
\end{equation}


The proof follows by taking the transpose of (\ref{eq:sym_proof_1}) and adding to (\ref{eq:sym_proof_1}).


\end{proof}

Theorem \ref{thm:sym} has also been noted in \cite{arora1} as a precursor to a special orthogonal initialisation. Fundamentally, theorem \ref{thm:sym} implies that for LNNs, the norm of the adjacent layers grow at the same rate. This creates a limit on the discrepancy between adjacent layers over time, and even arbitrary initialisations cannot make isolated layers to run away to divergence - this is captured in the following lemma:
 

\begin{lemma} For any arbitrary initialisation of weight matrices $W_l(t_0)$ and $W_{l+1}(t_0)$ for the LNN described in (\ref{eq:lnn}) at time $t_0$ and layers $l,l+1 \in \{1,\cdots,L\}$, there exists a constant $C_l(t_0)$ only dependent on initialisation of the weight matrices (and independent of $t$) such that $\|  |W_{l+1}(t)|_F - |W_{l}(t)|_F \| \leq C_{l}(t_0) \ \forall{t}$. Here $|W_l(t)|_F$ refers to the Frobenius norm.
\label{lem:bounded_diff}
\end{lemma}

\begin{proof}
Integrating both sides of theorem (\ref{thm:sym}) with time, we get:

\begin{equation} 
{W(t)}_{l} {W(t)}_{l}^T = {W(t)}_{l+1}^T {W(t)}_{l+1} + C
\label{eq:bounded_diff_proof_1}
\end{equation}

\noindent where $C$ is the constant of integration given by 

\begin{equation} 
C = {W(t_0)}_{l+1}^T {W(t_0)}_{l+1} - {W(t_0)}_{l} {W(t_0)}_{l}^T 
\label{eq:bounded_diff_proof_2}
\end{equation}

\noindent which is only dependent on initialisation of the weight matrices. Now consider the singular value decomposition of the weight matrices in (\ref{eq:bounded_diff_proof_1}) 

\begin{align*} 
 & U_l \Sigma_l V_l^T V_l \Sigma_l^T U_l^T  = V_{l+1} \Sigma_{l+1}^T U_{l+1}^T U_{l+1} \Sigma_{l+1} V_{l+1}^T  +  U_C \Sigma_C V_C^T \\
%
&\Rightarrow \Sigma_l  \Sigma_l^T = U_l^T V_{l+1} \Sigma_{l+1}^T \Sigma_{l+1} V_{l+1}^T U_l + U_l^T U_C \Sigma_C V_C^T U_l 
\end{align*}

\noindent Taking the 2-norm on both sides, noting that $|\Sigma_l  \Sigma_l^T|_2 = |W_l|_2^2$ and $U_l^T V_{l+1}$, $U_l^T U_C$ and $V_C^T U_l$ are rotations, 

\begin{align} 
 & |W_l|_2^2 \leq |W_{l+1}|_2^2 + |\Sigma_C|_2 \\
 \Rightarrow & ||W_l|_2  - |W_{l+1}|_2| \leq |\Sigma_C|_2.
\label{eq:bounded_diff_proof_5}
\end{align}

\noindent Invoking the equivalence of norms, 

\begin{equation} 
||W_l|_F  - |W_{l+1}|_F| \leq \lambda |\Sigma_C|_2
\label{eq:bounded_diff_proof_6}
\end{equation}

for some $\lambda > 0$ and setting $C_l(t_0) = \lambda |\Sigma_C|_2$, we have the required bound.

\end{proof}

\noindent As an immediate consequence, we get the corollary below:

\begin{corollary}
For any arbitrary initialisation of weight matrices $W_l(t_0)$ and $W_{k}(t_0)$ for the LNN described in (\ref{eq:lnn}) at time $t_0$ and layers $l,k \in \{1,\cdots,L\}$, there exists a constant $\kappa_{lk}(t_0)$ only dependent on initialisation of the weight matrices (and independent of $t$) such that $\|  |W_{k}(t)|_F - |W_{l}(t)|_F \| \leq \kappa_{l}(t_0) \ \forall{t}$
\end{corollary}

\begin{proof}
Since difference of norms of adjacent weight matrices are bounded by a constant $C_l(t_0)$ as shown in lemma \ref{lem:bounded_diff}, difference of norm of any two layers is also bounded.
\end{proof}

More interestingly, it can be observed from (\ref{eq:bounded_diff_proof_5}) that the initial constant $C_l(t_0)$ is dependent on the magnitude of the initial difference $|\Sigma_C|_2$. Depending on the initialisation of the weight matrices, $|\Sigma_C|_2$ can be very close to zero or exactly zero. Consequently, the norm of the weight matrices can be kept arbitrarily close to each other during training.

\begin{corollary}
There exists an initialisation called the orthogonal initialisation of adjacent weight matrices such that $|\Sigma_C|_F = 0$ in (\ref{eq:bounded_diff_proof_5}). Consequently, $|W_l(t)|_F = |W_k(t)|_F,  l,k \in \{1,\cdots,L\}, \forall(t)$.
\label{cor:ortho_init}
\end{corollary}

\begin{proof}
Consider the singular value decomposition (SVD) of $W_l(t_0)$ -


\begin{equation*} 
 W_l(t_0) = U_l(t_0) \Sigma_l(t_0) V_l^T(t_0)
\end{equation*}

Initialise $\Sigma_l(t_0) \Sigma_l(t_0)^T = \Sigma_{l+1}(t_0)^T \Sigma_{l+1}(t_0)$, $V_{l+1} (t_0) = U_l(t_0)$ $\forall{l}$. From (\ref{eq:bounded_diff_proof_2}),

\begin{align} 
 & |\Sigma_C|_F = |C|_F = |V_{l+1}^T(t_0) C V_{l+1}(t_0)|_F \nonumber \\
 & = |\Sigma_{l+1}^T(t_0)  \Sigma_{l+1}(t_0) -  \nonumber \\
 & V_{l+1}(t_0)^T U_l(t_0) \Sigma_{l}(t_0) \Sigma_{l}^T(t_0) U_l^T(t_0) V_{l}^T (t_0)|_F \nonumber \\
 & \leq | |\Sigma_{l+1}^T(t_0)  \Sigma_{l+1}(t_0)|_F - |\Sigma_{l}(t_0) \Sigma_{l}^T(t_0)|_F | = 0 \nonumber \\
 & \Rightarrow |\Sigma_C|_F = 0 \nonumber
\end{align}

\end{proof}

With the initialisation in corollary \ref{cor:ortho_init}, the weight matrices start from the same weight initially, and evolve exactly at the same pace so that their norms are the same at all times. For realistic starting points such as the \textit{Glorot initialisation} (\cite{glorot1}), we can expect $|\Sigma_C|_F \sim 0$.

\begin{corollary}
Let the weight matrices be Glorot initialised as in (\cite{glorot1}). Then, $\mathbb{E} [ | |W_l(t)|_F - |W_k(t)|_F|] = 0 ,  l,k \in \{1,\cdots,L\}, \forall(t)$.
\label{cor:glorot_init}
\end{corollary}

\begin{proof}

In the Glorot initialisation, elements of $W_l(t_0)$ are i.i.d sampled $\sim \mathcal{N}(0,(\sqrt{\frac{\beta}{n_l n_{l+1}})}), \beta > 0$.

\begin{align} 
 & \Rightarrow \mathbb{E} [W_l(t_0)  W_l(t_0)^T ] = \frac{\beta}{n_{l+1}} I_{n_{l+1}}, var(|W_{l}| _{F}^{2})=\frac{2\beta^{2}}{n_{l}n_{l+1}} \nonumber \\
 & \text{and}, \mathbb{E} [W_{l+1}^T(t_0)  W_{l+1}(t_0)^T ] = \frac{\beta}{n_{l+1}} I_{n_{l+1}}, var(|W_{l+1}| _{F}^{2})=\frac{2\beta^{2}}{n_{l+1}n_{l+2}}\nonumber \\
 & \Rightarrow \mathbb{E} [|\Sigma_C|_F] = 0, var(|\Sigma_{C}| _{F})\leq2\beta^{2}(\frac{1}{n_{l}n_{l+1}}+\frac{1}{n_{l+1}n_{l+2}})\nonumber 
\end{align}

From lemma \ref{cor:ortho_init}, $\mathbb{E} [ | |W_l(t)|_F - |W_k(t)|_F|] = 0$.

\end{proof}

Thus, realistic initialisation schemes, at least in LNNs, have the effect of evolving different layers at the same rate, so that there is no bottleneck layer, or one rogue layer does not destabilise the whole learning process. There are stronger guarantees for the evolution of the weight matrices if we consider particular cases of the loss function $l(\hat{y},y)$ in (\ref{eq:learning_algo}). For example, consider the $\mathbf{L}_2$ loss $l(\hat{y},y) = \frac{1}{2}|\hat{y}-y|^2$ in (\ref{eq:gd_ode_lnn}). 

\begin{corollary}
Suppose that compared to corollary \ref{cor:ortho_init}, we have $V_{l+1} (t_0) = U_l(t_0)$ $\forall{l}$, but $\Sigma_l(t_0) \Sigma_l(t_0)^T \neq \Sigma_{l+1}(t_0)^T \Sigma_{l+1}(t_0)$. Also, let $\Sigma_{yx} = U_{yy} S_{yx} V_{xx}^T$ be the SVD of $\Sigma_{yx}$ and $U_L = U_{yy}, V_1 = V_{xx}$. Then $\Sigma_l(t) \Sigma_l(t)^T \rightarrow \Sigma_{l+1}(t)^T \Sigma_{l+1}(t)$ as $t \rightarrow \infty, \forall l$.
\label{cor:ortho_init_2}
\end{corollary}

\begin{proof}
Let $\Sigma_l = diag\{ \sigma_{1,l},\cdots,\sigma_{K,l}\}$, $S_{yx} = diag\{ s_1,\cdots \}$, where $K = \min\{n_l,n_{l+1}\}$. Then, for the initialisation in corollary \ref{cor:ortho_init_2}, \cite{saxe1} has shown that for the $\mathbf{L}_2$ loss, the same modes in each layer influence each other independently of the other modes (the so called decoupling) to evolve as

\begin{equation}
\label{eq:sigma_saxe}
\tau \dot{\sigma_{k,l}} = \prod_{i \in [L]\setminus l} \sigma_{k,i} \Big( s_k -  \prod_{i \in [L]}\sigma_{k,i} \Big)
\end{equation}

Denoting $s_k -  \prod_{i \in [L]}\sigma_{k,i} = c_k$, it being a layer-independent constant for each mode $k$, we get, $\tau \dot{\sigma_{k,l}} = c_k \prod_{i \in \{[L]\setminus l}\sigma_{k,i}$. Let us now consider the case when $\sigma_{k,l}>0$ for all $l$. Denoting $d_k = \prod_{i \in [L]}\sigma_{k,i}$, in this case, from (\ref{eq:sigma_saxe}), we can write $\frac{\dot{\sigma_{k,l}}}{\dot{\sigma_{k,l'}}} = \frac{c_k d_k}{\sigma_{k,l}} \frac{\sigma_{k,l'}}{c_k d_k} =  \frac{\sigma_{k,l'}}{\sigma_{k,l}}$. Suppose wlog  $\frac{\sigma_{k,l'}(t_0)}{\sigma_{k,l}(t_0)} = r_{k, l,l'} > 1$. Then, 

\begin{align*}
& \frac{\sigma_{k,l'}(t_0+\delta t)}{\sigma_{k,l}(t_0+\delta t)} = \frac{\sigma_{k,l'}(t_0) + \dot{\sigma_{k,l'}}(t_0) \delta t}{\sigma_{k,l}(t_0) + \dot{\sigma_{k,l}}(t_0) \delta t}\\\nonumber
&= \frac{r_{k,l,l'}\sigma_{k,l}(t_0) + \frac{\dot{\sigma_{k,l}}(t_0) \delta t}{r_{k,l,l'}}}{\sigma_{k,l}(t_0) + \dot{\sigma_{k,l}}(t_0) \delta t} = \hat{r}_{k,l,l'} < r_{k,l,l'}
\end{align*}

For $\delta \rightarrow 0$, $1 \leq \hat{r}_{k,l,l'} < r_{k,l,l'} $, and repeating the argument $\forall t$ the ratio will smoothly approach 1, $\Rightarrow \Sigma_l(t) \Sigma_l(t)^T \rightarrow \Sigma_{l+1}(t)^T \Sigma_{l+1}(t)$ as $t \rightarrow \infty, \forall l$.

Clearly, $1 \leq \hat{r}_{k,l,l'} < r_{k,l,l'} $. Taking $\delta \rightarrow 0$ and repeating the argument $\forall t$, $ \frac{\hat{r}_{k,l,l'}}{r_{k,l,l'}} \rightarrow 1 $ $\Rightarrow \Sigma_l(t) \Sigma_l(t)^T \rightarrow \Sigma_{l+1}(t)^T \Sigma_{l+1}(t)$ as $t \rightarrow \infty, \forall l$.

\end{proof}

Additionally, we can show that in case of $\mathbf{L}_2$ loss the update equation for the $W_l$-th matrix in (\ref{eq:gd_ode_lnn}) now becomes -




\begin{equation} 
\tau \frac{d W_l}{dt} = \Big[\prod_{i=l+1}^{L } W_i \Big]^T \Big[\Sigma_{yx} - \Big(\prod _{i=1}^{L} W_l \Big) \Big]  \Big[\prod_{i=1}^{l -1} W_i \Big]^T.
\label{eq:gd_ode_lnn_l2_whitened}
\end{equation}

\noindent where $\Sigma_{yx} = \frac{1}{m} \sum_{i=1}^m y_i x_i^T$ is the \textit{input-output covariance} matrix, and the input has been assumed to have been \textit{whitened} (i.e., $\Sigma_{xx} = \frac{1}{m} \sum_{i=1}^m x_i x_i^T= I_d$ ). 


It is easy to infer about the upper bound on the rate of growth of the norm of the weight matrices from (\ref{eq:gd_ode_lnn_l2_whitened}) -

\begin{equation} 
\tau \frac{d |W_l|_F}{dt} \leq \tau \Big\vert \frac{d W_l}{dt}\Big\vert_F
\leq \Big\vert \prod_{i=1,i \neq l}^L W_i \Big\vert \Big\vert \Sigma_{yx} - \Big(\prod _{j=1}^{L} W_j \Big) \Big\vert_F
\label{eq:ub_ode_lnn_l2}
\end{equation}

%
%
%
%
%
%
%

%

For a general LNN with arbitrary initialisation, it is difficult to arrive at an exact rate of evolution of the norm of weight matrices. However, it is imperative from (\ref{eq:gd_ode_lnn_l2_whitened}) that $\prod_{i=1}^L W_l$ reaches $\Sigma_{yx}$ at convergence. Therefore, we can study the evolution of a modified upper bound in (\ref{eq:ub_ode_lnn_l2}), which will allow us to bound the fastest possible time for the LNN to reach the solution. This is captured by the following lemma - 

\begin{theorem} Assume that the LNN described in (\ref{eq:lnn}) has a reasonable initialisation of weight matrices such as given by corollary \ref{cor:glorot_init} for  at time $t_0$. Also assume $\triangle \doteq \arg \underset{l,k \in L}{\max} ||W_l|_F - |W_k|_F | $ be the maximum difference between any two layer norms, which is determined at initialisation (and invariant over time). Define

\begin{equation}
U = (|W_l|_F + \triangle)^L 
\label{eq:defn_U}
\end{equation}

Then there exists constants $\kappa_1 > 0$ and  $ 0 \leq \kappa_2 \leq 1$ such that the combined and individual layer strengths evolve at a maximum rate of


\begin{equation} 
\tau  \frac{d U}{dt} \leq L U^{2(1-\frac{1}{L})} [(2 \kappa_1 + 1) |\Sigma_{yx}|_F - \kappa_2 U] 
\label{eq:defn_evolve_U_u}
\end{equation}.

Also, for very deep LNNs where $L \rightarrow \infty$, $t_U$ = time required to reach a combined mode strength $U$ from initial state $U_0$, or alternatively, time required for individual layer norms to reach $O(U^{\frac{1}{L}})$ from an initial point of $O(U_0^{\frac{1}{L}})$ is given by 

\begin{equation} 
t_U \geq \frac{\kappa_2 \tau}{L} \Big[ \frac{1}{M^2} \log{\frac{U(M - U_0)}{U_0(M - U)} } - \frac{1}{M U} + \frac{1}{M U_0} \Big]
\label{eq:defn_t_U}
\end{equation}

where $M = (2 \kappa_1 + 1) |\Sigma_{yx}|_F$.

\label{lem:fastest_time_lnn}
\end{theorem}

\begin{proof}

For realistic initialisation as in corollary \ref{cor:glorot_init}, the norm of weight matrices start from a small value ($\sim 0$) and gradually grow over time during training. Hence we can reasonably assume that over the course of training, there exists a (small) $\kappa_1 > 0$ such that $|\Sigma_{yx}|_F \geq \kappa_1 |\prod_{i=1}^L W_i|_F, t \in T$. Here $T$ is the total evolution time.

Similarly, if $0 \leq \kappa_2 \doteq \arg \underset{t \in T}{\min} \frac{ |\prod_{i=1}^L W_i|_F}{(|W_l|_F + \triangle)^L} \leq 1 $, then at least towards convergence we can expect $\kappa_2 \sim 1$. Now,

\begin{align} 
& U = (|W_l|_F + \triangle)^L \Rightarrow |W_l|_F = U^{\frac{1}{L}} - \triangle \nonumber \\
& \Rightarrow \frac{d |W_l|_F}{dt} = \frac{d U^{\frac{1}{L}}}{dt}
\label{eq:evol_gd_1}
\end{align}

From (\ref{eq:ub_ode_lnn_l2}), 

\begin{align} 
& \tau \frac{d |W_l|_F}{dt} \leq \Big\vert \prod_{i=1,i \neq l}^L W_i \Big\vert \Big\vert \Sigma_{yx} - \prod _{j=1}^{L} W_j  \Big\vert_F \nonumber \\
& \Rightarrow \tau  \frac{d U^{\frac{1}{L}}}{dt} \leq U^\frac{L-1}{L} \Big\vert \Sigma_{yx} - \prod _{j=1}^{L} W_j \Big\vert_F
\label{eq:evol_gd_2}
\end{align}

%

A simple algebraic manipulation shows that  

\begin{align} 
& \Big\vert \Sigma_{yx} - \prod _{j=1}^{L} W_j \Big\vert_F \leq  (2 \kappa_1 + 1) |\Sigma_{yx}|_F - \Big\vert \prod _{j=1}^{L} W_j \Big\vert_F \nonumber \\
& \leq  (2 \kappa_1 + 1) |\Sigma_{yx}|_F - \kappa_2 U \nonumber
\label{eq:evol_gd_4}
\end{align}

From (\ref{eq:evol_gd_2}),

\begin{align} 
& \tau  \frac{d U^{\frac{1}{L}}}{dt} \leq U^\frac{L-1}{L} [(2 \kappa_1 + 1) |\Sigma_{yx}|_F - \kappa_2 U]  \\
& \Rightarrow \tau  \frac{d U}{dt} \leq L U^{2(1-\frac{1}{L})} [(2 \kappa_1 + 1) |\Sigma_{yx}|_F - \kappa_2 U]
\label{eq:evol_gd_5}
\end{align}

which is claim (\ref{eq:defn_evolve_U_u}). 
For very deep networks where $L \rightarrow \infty$, (\ref{eq:evol_gd_5}) reduces to

%

\begin{equation} 
\tau  \frac{d U}{dt} \leq L U^{2} [(2 \kappa_1 + 1) |\Sigma_{yx}|_F - \kappa_2 U]
\label{eq:evol_gd_7}
\end{equation}

which when integrated, implies (\ref{eq:defn_t_U}).

%

\end{proof}

Theorem \ref{lem:fastest_time_lnn} suggests that the fastest time to convergence, for a small and fixed learning rate (such that the ODE in (\ref{eq:gd_ode_lnn_l2_whitened}) is a valid approximation to the discrete evolution via gradient descent) does improve with increasing depth, and can reach a very small value for a very deep network. This is a generalisation of the orthogonally initialised network in \cite{saxe1}. 

\section{Linearisation of Neural Networks}\label{sec:gnn}

In this section, we will describe a linearisation technique for applying the insights gained from section \ref{sec:lnn} in understanding the behaviour of layer growth in some forms of GNNs.

Consider the GNN in (\ref{eq:gnn}), and specifically focus on the case of no bias and $\sigma$ being the RelU function; a corresponding time-dependent matrix form can be written for (\ref{eq:gnn}) for one individual training point $\{ x_i,y_i \}$ as 

\begin{equation} 
\hat{y}_i = \phi(D_{L,i,t} W_L(t) D_{L-1,i,t} W_{L-1}(t) \cdots D_{1,i,t} W_1(t) x_i)
\label{eq:relu}
\end{equation}

where in (\ref{eq:relu}), $D_{l,i,t}$ is a diagonal $n_{l+1} \times n_{l+1}$ matrix capturing the time dependent effect of RelU for sample $i$. That is, the $j$-th diagonal entry of $D_{l,i,t}$ is $0$ or $1$ depending on whether the $j$-th output of $W_{l}(t) D_{l-1,i,t} W_{l-1}(t) \cdots D_{1,i,t} W_1(t) x_i$ is $\leq 0$ or $>0$ respectively. Let $\phi$ be the softmax function in (\ref{eq:relu}). 

\begin{figure*}[!t]
\centering
\includegraphics[width=\textwidth]{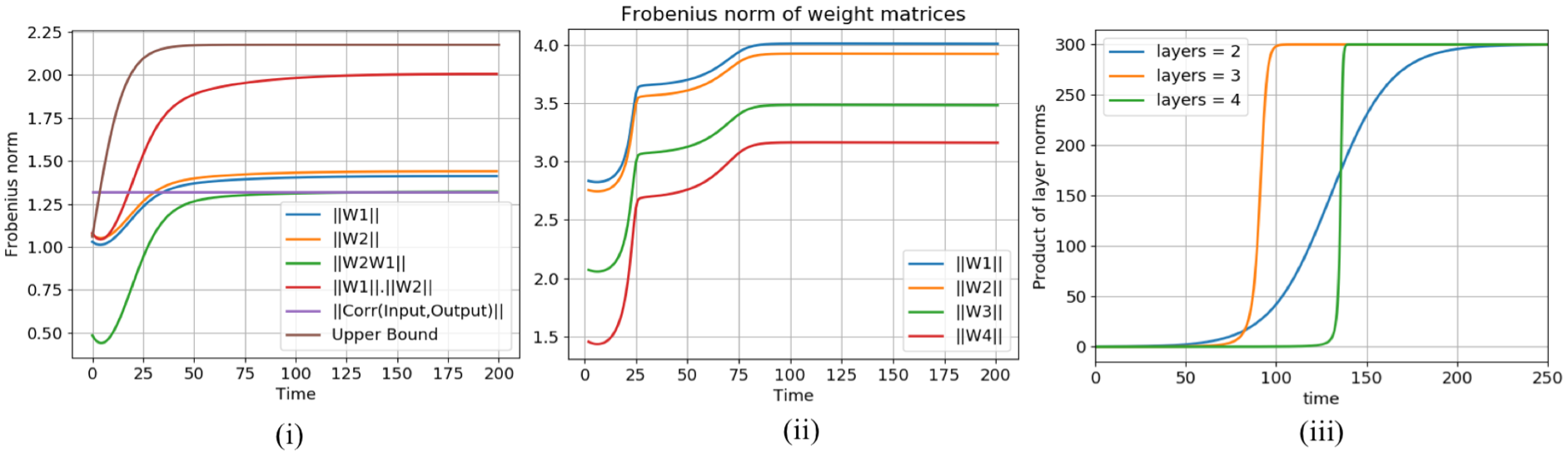}
\caption{Growth of layer norms in a LNN (please view in colour). (i) In a 2 layer scenario, individual layers grow at the same rate, staying equidistant to each other. The upper bound (in brown) dominates closely the product of norms (in red) (ii) In a 4 layer scenario, layers again grow at the same rate (iii) Analysis of upper bound in (\ref{eq:defn_evolve_U_u}) for 2,3 and 4 layers show the effect of depth in attaining convergence} 
\label{fig:lnn}
\end{figure*}

In order to apply insights from LNN to (\ref{eq:relu}), a reasonable assumption is that inside a time window $[ t - \triangle t, t + \triangle t]$, the behaviour of the network is constant wrt the RelU, that is, $D_{l,i,t}$ is constant $\doteq \bar{D}_{l,i}$. $\bar{D}_{l,i}$ denotes the \textit{piecewise linearisation} of a nonlinear RelU network along the time and sample dimensions, and allows us to apply LNN theory to nonlinear networks. Then the corresponding form of (\ref{eq:gd_ode_lnn}) for sample $i$, with $l$ as the cross-entropy loss becomes 

\begin{equation} 
\tau \bar{D}_{li} \frac{d W_l}{dt} = - \Big[\prod_{j=l+1}^{L } \bar{D}_{ji} W_j \Big]^T \Big[ (y_i - o_i)] x_i^T  \Big]  \Big[\prod_{j=1}^{l -1} \bar{D}_{ji} W_j \Big]^T.
\label{eq:gd_ode_relu}
\end{equation}


\begin{figure}[!h]
\centering
\includegraphics[width=3.2 in]{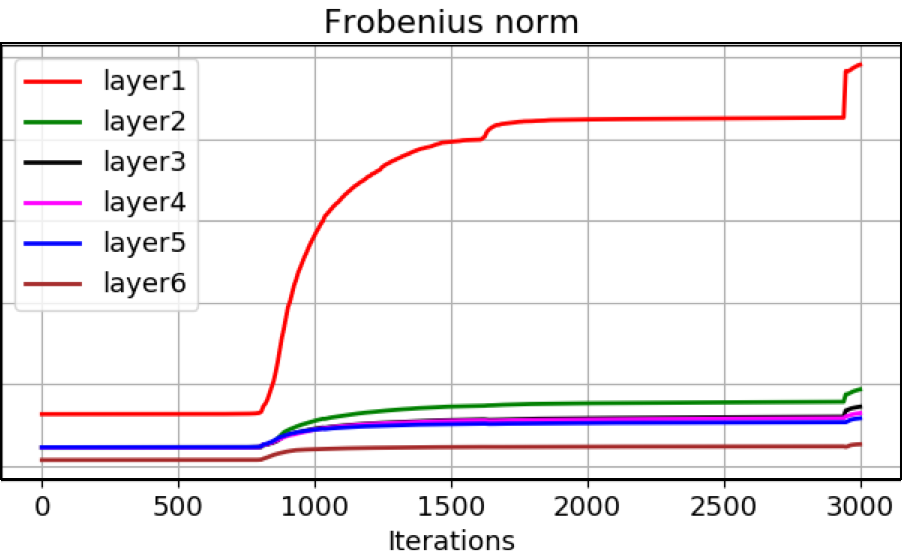}
\caption{Evolution of layer norms in a 6 layer RelU network learning the MNIST dataset (please view in colour); 1 denotes closest to the input. The rate of growth of the upper layers (4-6) is similar, specially towards the end of training} 
\label{fig:gnn}
\end{figure}

where $o_i = D_{L,i,t} W_L(t) \cdots D_{1,i,t} W_1(t) x_i$. In similar lines to the proof in theorem \ref{thm:sym}, right multiplying both sides of the $l$-th equation of (\ref{eq:gd_ode_relu}) by $W_l^T \bar{D}_{li}^T$ and left multiplying both sides of the $l+1$-th equation of (\ref{eq:gd_ode_relu}) by $W_{l+1}^T \bar{D}_{l+1,i}^T$, taking transpose followed by adding to both sides, and recalling that we are treating $\bar{D}_{li}$ as a constant inside a time window, we can show

\begin{align} \label{eq:theorem_relu}
& \frac{d }{dt} (W_{l+1}^T \bar{D}_{l+1,i}^T \bar{D}_{l+1,i} W_{l+1}) =  \frac{d}{dt} (\bar{D}_{li} W_{l} W_{l}^T \bar{D}_{li}^T), \nonumber \\
& t \in [ t - \triangle t, t + \triangle t]
\end{align}

Note that $\bar{W}_{li} = \bar{D}_{li} W_l$ is a modified weight matrix specific to sample $i$ with those rows of $W_l$ removed where there are zeros in the diagonal of $\bar{D}_{li}$. Also, note that, the original weight matrix $W_l$ can be written as

\begin{align} 
& W_l = \frac{1}{m} \sum_{i=1}^m \bar{W}_{li} = [\frac{1}{m} \sum_{i=1}^m \bar{D}_{li}] W_l \\
& \Rightarrow \frac{d}{dt} (W_l W_l^T) = \frac{1}{m^2} \sum_{i,j} \frac{d}{dt}(\bar{D}_{li} W_l W_l^T \bar{D}_{lj}^T) \\
&\Rightarrow \frac{d}{dt} (W_{l+1}^T W_{l+1}) - \frac{d}{dt} (W_l W_l^T) = \nonumber \\
&\frac{1}{m^2-m} \frac{d}{dt} \Big[ \sum_{i \neq j} \Big( \bar{W}_{l+1,i}^T \bar{W}_{l+1,j} - \bar{W}_{li} \bar{W}_{lj}^T \Big) \Big] \neq [0]
\label{eq:weight_decompose}
\end{align}

\noindent after cancelling the equal terms for the same samples using (\ref{eq:theorem_relu}).

%

%

Thus, unlike LNNs in theorem \ref{thm:sym}, the rate of growth of the norm of adjacent layers in the non-linear case is not guaranteed to be equal, even within a small time window, because of the coupled effects of the RelU from different samples. Each sample tries to modify the partial norm of adjacent layers at a similar rate (refer (\ref{eq:theorem_relu})), but coupling from different samples breaks down the symmetry observed in LNNs.

However, it is known that a multilayer neural net progressively combine features into common groups as information passes through the network \cite{MahendranV14}. This means that 2 samples $i$ and $j$ such that $y_i = y_j$ (same labels) will after some layer $l$ and time $t$, have similar behaviour wrt the RelU for all layers $\acute{l} > l,\acute{t} > t$; these 2 samples will be similarly classified after the $l$-th layer after some training time ($\Rightarrow \bar{D}_{\acute{l},i} = \bar{D}_{\acute{l},j} \forall \acute{l} > l,\acute{t} > t$) and will cancel out in the RHS in (\ref{eq:theorem_relu}).

This implies that the sum in RHS in (\ref{eq:weight_decompose}), which is $O(m^2-m)$ at the beginning of training when the weights are random, will be some $O(m^2-\alpha_l(t))$ as training progresses, such that $\alpha_{l+1}(t) \geq \alpha_l(t)$ and $\alpha_l$ denotes the ever increasing number of input samples whose coupled contributions cancel out following (\ref{eq:theorem_relu}). This implies that growth rates of higher layers progressively match as training time progresses - this is captured in the following conjecture.

\begin{conj} \label{conj:gnn_growth}
When training a RelU network described in (\ref{eq:relu}), $\vert \frac{d}{dt} (W_{l+1}^T W_{l+1}) - \frac{d}{dt} (W_l W_l^T) \vert_F \rightarrow 0$ as $l \rightarrow L$, $t \rightarrow \infty$ for some $l > l^{*}$.
\end{conj}

\section{Experiments}\label{sec:expts}

In this section, we empirically validate the insights we gained in sections \ref{sec:lnn} and \ref{sec:gnn} by training LNNs on synthetically generated data and observe training performance of a RelU network on MNIST \cite{mnist}.

For the experiments with LNN, we generate 10-dimensional samples from a mixture of 3 Gaussians with 400 samples from each Gaussian. Each Gaussian has a randomly scaled and rotated covariance. We generate the 3-dimensional predictions (ground truth) by multiplying the sample with a randomly initialised $3 \times 10$ matrix. During training, we try to learn the linear transform by a 3 layer and a 5 layer network respectively. The dimension of the weights of the 3 layer network is $6 \times 10 $ for $W_1$ and $3 \times 6$ for $W_2$. The corresponding dimensions for the 5 layer network is $8 \times 10 $ for $W_1$, $6 \times 8 $ for $W_2$, $4 \times 6 $ for $W_3$ and $3 \times 4 $ for $W_4$. All matrices have been Glorot initialised.

Figure \ref{fig:lnn} (i) shows the growth of norms of $W_1$ (blue) and $W_2$ (orange) with the same rate in the 2 layer case, always keeping within the distance that they started from at initialisation. Reading and deriving from figure \ref{fig:lnn} (i), $\kappa_1 = 0.98$ and $\kappa_2 = 0.65$ for (\ref{eq:defn_evolve_U_u}). $|W_2 W_1|_F$ (green) approaches $|\Sigma_{yx}|_F$ (violet) at convergence, while $|W_1|_F |W_2|_F$ (red) is dominated by the upper bound from (\ref{eq:defn_evolve_U_u}) (brown).

Figure \ref{fig:lnn} (ii) shows the growth of norms of $W_1 - W_4$ for the 5 layer case, where the ground truth settings are similar to Figure \ref{fig:lnn} (i), thus showing that the layers indeed grow at the same pace. Figure \ref{fig:lnn} (iii) shows the utility of the upper bound from (\ref{eq:defn_evolve_U_u}). With identical settings of $\kappa_1$, $\kappa_2$, $W_l(t_0)$ and $|\Sigma_{yx}|_F$ in (\ref{eq:defn_evolve_U_u}) and by only varying the number of layers (from 2 to 4), we observe that smaller networks have a slower but smoother transition to convergence, whereas larger layers have periods of low learning followed by a sharp transition and explosion in growth. This suggests that smaller networks might be inherently stable at the cost of slower learning.

We also experiment to test our insights from section \ref{sec:gnn}. We train a 7 layer fully connected RelU network on the MNIST dataset \cite{mnist}, with the number of hidden units in layers 1 to 6 all set to 100. As can be seen from figure \ref{fig:gnn}, the norm of the layers have a similar rate of growth towards the end of training.

Note that the initial part of the layer growth till 1600 has different layers growing at different speed, but after that the upper layers, specially layers 4-6 seem to follow a similar growth rate. This seems to support conjecture \ref{conj:gnn_growth} for this particular example, but more analysis and experimental validation is needed to prove conjecture \ref{conj:gnn_growth} firmly.

\section{Discussion and Conclusion}\label{sec:concl}


In this work, we analyse the behaviour of layer growth in general linear neural networks, where the coupling between adjacent layer weights leads to several interesting phenomenon that can also be observed in nonlinear nets.

In particular, we show that irrespective of the layer dimensions, reasonable initialisations make the layers in a LNN grow at approximately the same rate. Increasing layers have a compounding effect in the rate of layer growth. More layers lead to starkly separated growth phases in training; periods of low growth are followed by rapid growth to convergence. However, fewer layers might lead to a more stable but slower learning.

We also demonstrate a novel linearisation technique to apply insights from LNN to general nets; we show that individual samples force adjacent layers to grow at similar rates, but the nonlinearity breaks down overall symmetry in the growth of adjacent layers. We also hypothesise that upper layers should grow at similar rates in a GNN as training progresses.

Several interesting directions present themselves as a followup - replacing fully connected layers by CNNs, understanding the effect of batch processing on learning dynamics as well as developing a more well rounded theory for linearisation will be pursued for further development. 

\bibliographystyle{named}
\bibliography{ijcai_layerdynamics}

\begin{thebibliography}{}

\bibitem[\protect\citeauthoryear{{Advani} and {Saxe}}{2017}]{advani1}
M.~S. {Advani} and A.~M. {Saxe}.
\newblock {High-dimensional dynamics of generalization error in neural
  networks}.
\newblock {\em ArXiv e-prints}, October 2017.

\bibitem[\protect\citeauthoryear{Arora \bgroup \em et al.\egroup
  }{2018}]{arora1}
Sanjeev Arora, Nadav Cohen, and Elad Hazan.
\newblock On the optimization of deep networks: Implicit acceleration by
  overparameterization.
\newblock {\em CoRR}, abs/1802.06509, 2018.

\bibitem[\protect\citeauthoryear{Baldi and Hornik}{1989}]{hornik1}
P.~Baldi and K.~Hornik.
\newblock Neural networks and principal component analysis: Learning from
  examples without local minima.
\newblock {\em Neural Netw.}, 2(1):53--58, January 1989.

\bibitem[\protect\citeauthoryear{Chaudhari and Soatto}{2017}]{ode2}
Pratik Chaudhari and Stefano Soatto.
\newblock Stochastic gradient descent performs variational inference, converges
  to limit cycles for deep networks.
\newblock {\em CoRR}, abs/1710.11029, 2017.

\bibitem[\protect\citeauthoryear{Choromanska \bgroup \em et al.\egroup
  }{2014}]{Choromanska15a}
Anna Choromanska, Mikael Henaff, Micha{\"{e}}l Mathieu, G{\'{e}}rard~Ben Arous,
  and Yann LeCun.
\newblock The loss surface of multilayer networks.
\newblock {\em CoRR}, abs/1412.0233, 2014.

\bibitem[\protect\citeauthoryear{Fukumizu}{1998}]{fukumizu1}
Kenji Fukumizu.
\newblock Effect of batch learning in multilayer neural networks.
\newblock In {\em Proceedings of the Fifth International Conference on Neural
  Information Processing}, 1998.

\bibitem[\protect\citeauthoryear{Glorot and Bengio}{2010}]{glorot1}
Xavier Glorot and Yoshua Bengio.
\newblock Understanding the difficulty of training deep feedforward neural
  networks.
\newblock In Yee~Whye Teh and Mike Titterington, editors, {\em Proceedings of
  the Thirteenth International Conference on Artificial Intelligence and
  Statistics}, volume~9 of {\em Proceedings of Machine Learning Research},
  pages 249--256, Chia Laguna Resort, Sardinia, Italy, 13--15 May 2010. PMLR.

\bibitem[\protect\citeauthoryear{Goodfellow \bgroup \em et al.\egroup
  }{2016}]{Goodfellow-et-al-2016}
Ian Goodfellow, Yoshua Bengio, and Aaron Courville.
\newblock {\em Deep Learning}.
\newblock MIT Press, 2016.
\newblock \url{http://www.deeplearningbook.org}.

\bibitem[\protect\citeauthoryear{Greff \bgroup \em et al.\egroup }{2015}]{lstm}
Klaus Greff, Rupesh~Kumar Srivastava, Jan Koutn{\'{\i}}k, Bas~R. Steunebrink,
  and J{\"{u}}rgen Schmidhuber.
\newblock {LSTM:} {A} search space odyssey.
\newblock {\em CoRR}, abs/1503.04069, 2015.

\bibitem[\protect\citeauthoryear{Hardt and Ma}{2016}]{HardtM16}
Moritz Hardt and Tengyu Ma.
\newblock Identity matters in deep learning.
\newblock {\em CoRR}, abs/1611.04231, 2016.

\bibitem[\protect\citeauthoryear{He \bgroup \em et al.\egroup }{2015}]{resnet}
Kaiming He, Xiangyu Zhang, Shaoqing Ren, and Jian Sun.
\newblock Deep residual learning for image recognition.
\newblock {\em CoRR}, abs/1512.03385, 2015.

\bibitem[\protect\citeauthoryear{Kawaguchi}{2016}]{kawaguchi1}
Kenji Kawaguchi.
\newblock Deep learning without poor local minima.
\newblock In D.~D. Lee, M.~Sugiyama, U.~V. Luxburg, I.~Guyon, and R.~Garnett,
  editors, {\em Advances in Neural Information Processing Systems 29}, pages
  586--594. Curran Associates, Inc., 2016.

\bibitem[\protect\citeauthoryear{Krizhevsky \bgroup \em et al.\egroup
  }{2012}]{alexnet}
Alex Krizhevsky, Ilya Sutskever, and Geoffrey~E Hinton.
\newblock Imagenet classification with deep convolutional neural networks.
\newblock In F.~Pereira, C.~J.~C. Burges, L.~Bottou, and K.~Q. Weinberger,
  editors, {\em Advances in Neural Information Processing Systems 25}, pages
  1097--1105. Curran Associates, Inc., 2012.

\bibitem[\protect\citeauthoryear{{Lecun} \bgroup \em et al.\egroup
  }{1998}]{mnist}
Y.~{Lecun}, L.~{Bottou}, Y.~{Bengio}, and P.~{Haffner}.
\newblock Gradient-based learning applied to document recognition.
\newblock {\em Proceedings of the IEEE}, 86(11):2278--2324, Nov 1998.

\bibitem[\protect\citeauthoryear{Li \bgroup \em et al.\egroup }{2015}]{ode1}
Qianxiao Li, Cheng Tai, and Weinan E.
\newblock Dynamics of stochastic gradient algorithms.
\newblock {\em CoRR}, abs/1511.06251, 2015.

\bibitem[\protect\citeauthoryear{Mahendran and Vedaldi}{2014}]{MahendranV14}
Aravindh Mahendran and Andrea Vedaldi.
\newblock Understanding deep image representations by inverting them.
\newblock {\em CoRR}, abs/1412.0035, 2014.

\bibitem[\protect\citeauthoryear{Poggio and Liao}{2017}]{PoggioL17}
Tomaso~A. Poggio and Qianli Liao.
\newblock Theory {II:} landscape of the empirical risk in deep learning.
\newblock {\em CoRR}, abs/1703.09833, 2017.

\bibitem[\protect\citeauthoryear{Poggio \bgroup \em et al.\egroup
  }{2018a}]{Poggio_puzzle}
Tomaso~A. Poggio, Kenji Kawaguchi, Qianli Liao, Brando Miranda, Lorenzo
  Rosasco, Xavier Boix, Jack Hidary, and Hrushikesh Mhaskar.
\newblock Theory of deep learning {III:} explaining the non-overfitting puzzle.
\newblock {\em CoRR}, abs/1801.00173, 2018.

\bibitem[\protect\citeauthoryear{Poggio \bgroup \em et al.\egroup
  }{2018b}]{Poggio_puzzle_2}
Tomaso~A. Poggio, Qianli Liao, Brando Miranda, Andrzej Banburski, Xavier Boix,
  and Jack Hidary.
\newblock Theory iiib: Generalization in deep networks.
\newblock {\em CoRR}, abs/1806.11379, 2018.

\bibitem[\protect\citeauthoryear{Russakovsky \bgroup \em et al.\egroup
  }{2014}]{imagenet-1k}
Olga Russakovsky, Jia Deng, Hao Su, Jonathan Krause, Sanjeev Satheesh, Sean Ma,
  Zhiheng Huang, Andrej Karpathy, Aditya Khosla, Michael~S. Bernstein,
  Alexander~C. Berg, and Fei{-}Fei Li.
\newblock Imagenet large scale visual recognition challenge.
\newblock {\em CoRR}, abs/1409.0575, 2014.

\bibitem[\protect\citeauthoryear{Saxe \bgroup \em et al.\egroup }{2013}]{saxe1}
Andrew~M. Saxe, James~L. McClelland, and Surya Ganguli.
\newblock Exact solutions to the nonlinear dynamics of learning in deep linear
  neural networks.
\newblock {\em CoRR}, abs/1312.6120, 2013.

\bibitem[\protect\citeauthoryear{{Soudry} and {Carmon}}{2016}]{Soudry1}
Daniel {Soudry} and Yair {Carmon}.
\newblock {No bad local minima: Data independent training error guarantees for
  multilayer neural networks}.
\newblock {\em arXiv e-prints}, page arXiv:1605.08361, May 2016.

\bibitem[\protect\citeauthoryear{{Tachet des Combes} \bgroup \em et al.\egroup
  }{2018}]{Combes2018}
R.~{Tachet des Combes}, M.~{Pezeshki}, S.~{Shabanian}, A.~{Courville}, and
  Y.~{Bengio}.
\newblock {On the Learning Dynamics of Deep Neural Networks}.
\newblock {\em ArXiv e-prints}, September 2018.

\end{thebibliography}

\end{document}